\documentclass{article}

\usepackage[final]{nips_2018}

\usepackage[utf8]{inputenc} % allow utf-8 input
\usepackage[T1]{fontenc}    % use 8-bit T1 fonts
\usepackage{hyperref}       % hyperlinks
\usepackage{url}            % simple URL typesetting
\usepackage{booktabs}       % professional-quality tables
\usepackage{amsfonts}       % blackboard math symbols
\usepackage{nicefrac}       % compact symbols for 1/2, etc.
\usepackage{microtype}      % microtypography
\usepackage{graphicx}
\usepackage{subfigure}
\usepackage{times}
\usepackage{latexsym}
\usepackage{amsfonts}
\usepackage{amsmath}
\usepackage{amssymb,amsthm}
\usepackage{bbm}
\usepackage{lipsum}
\usepackage[shortlabels]{enumitem}
\usepackage{color}
\usepackage{array}
\usepackage[font={footnotesize}]{caption}
\usepackage{enumitem}
\usepackage{floatrow}

\title{Mapping Images to Scene Graphs with Permutation-Invariant Structured
Prediction}

\author{
  Roei Herzig\thanks{Equal Contribution.}\\
  Tel Aviv University \\
  \texttt{roeiherzig@mail.tau.ac.il} \\
  \And
  Moshiko Raboh$^*$\\
  Tel Aviv University \\
  \texttt{mosheraboh@mail.tau.ac.il} \\
  \AND
  Gal Chechik \\
  Bar-Ilan University, NVIDIA Research \\
  \texttt{gal.chechik@biu.ac.il} \\
  \And
  Jonathan Berant \\
  Tel Aviv University, AI2 \\
  \texttt{joberant@cs.tau.ac.il} \\
  \And
  Amir Globerson \\
  Tel Aviv University \\
  \texttt{gamir@post.tau.ac.il} \\
}

\usepackage{bbold}

\renewcommand{\xi}{{\xx}^{(m)}}

\newcommand{\cF}{\mathcal{F}}
\newcommand{\needcite}[1]{}

\newcommand{\be}{\begin{equation}}
\newcommand{\ee}{\end{equation}}
\newcommand{\benn}{\begin{equation*}}
\newcommand{\eenn}{\end{equation*}}
\newcommand{\bea}{\begin{eqnarray*}}
\newcommand{\eea}{\end{eqnarray*}}
\newcommand{\bean}{\begin{eqnarray}}
\newcommand{\eean}{\end{eqnarray}}

\newcommand{\ww}{\boldsymbol{w}} 
\newcommand{\xx}{\boldsymbol{x}} 
 
\newcommand{\yy}{\boldsymbol{y}} 
\renewcommand{\ss}{\boldsymbol{s}} 
\newcommand{\zz}{\boldsymbol{z}}

\newcommand{\glf}{{\mathcal{F}}}

\newcommand{\ve}{\zz}
\newcommand{\vescalar}{z}
\newcommand{\onehotvector}[1]{\mathbb{1}\left[#1\right]}

\newcommand{\ignore}[1]{}
\newcommand{\comment}[1]{}

\newcommand{\phiv}{\boldsymbol{\phi}}
\newcommand{\alphav}{\boldsymbol{\alpha}}

\newcommand{\polyring}[1]{\reals\left[x_1,\ldots,x_n\right]}
\newcommand{\triplet}[3]{$\big\langle$#1, #2, #3$\big\rangle $}

\definecolor{atomictangerine}{rgb}{0.8, 0.2, 0.1}
\definecolor{turq}{rgb}{0.0, 0.5, 0.5}
\definecolor{darkturq}{rgb}{0.0, 0.4, 0.4}
\definecolor{bright}{rgb}{0.8, 0.1, 0}
\definecolor{darkgray}{gray}{0.3}
\definecolor{mahogany}{rgb}{0.6, 0.05, 0.05}
\definecolor{pink}{rgb}{1,0.05,0.6}
\definecolor{myblue}{rgb}{0.3,0.05,0.9}

\newcommand\jb[1]{\textcolor{bright}{(\textbf{JB:} #1 )}}

{
\begin{center}
\begin{boxedminipage}{0.8\linewidth}
\begin{center}
\textbf{\texttt{#1}}
\end{center}
\rm
\begin{tabbing}
....\=...\=...\=...\=...\=  \+ \kill
} %
{\end{tabbing} 
\end{boxedminipage} \end{center} %\vspace{0.3cm}
}

\renewcommand{\eqref}[1]{Eq.~\ref{#1}}

\newcommand{\figref}[1]{Figure \ref{#1}}
\newcommand{\secref}[1]{Section \ref{#1}}

\newcommand{\reals}{\mathbb{R}}

\newtheorem{theorem}{Theorem}

\newtheorem{definition}{Definition}

\begin{document}

\maketitle

% ---------------
% Sections
% ---------------
\begin{abstract}
Machine understanding of complex images is a key goal of artificial intelligence. One challenge underlying this task is that visual scenes contain multiple inter-related objects, and that global context plays an important role in interpreting the scene. 
A natural modeling framework for capturing such effects is structured prediction,
which optimizes over complex labels, while modeling within-label interactions. However, it is unclear what principles should guide 
the design of a structured prediction model that utilizes the power of deep learning components.
Here we propose a design principle for such architectures that follows from a natural requirement of permutation invariance. We prove a necessary and sufficient characterization for architectures that follow this invariance, and discuss its implication on model design. Finally, we show that the resulting model achieves new state-of-the-art results on the \textit{Visual Genome} scene-graph labeling benchmark, outperforming all recent approaches.

\end{abstract}
\section{Introduction}
\label{introduction}
Understanding the semantics of a complex visual scene is a fundamental problem in machine perception. It often requires recognizing multiple objects in a scene, together with their spatial and functional relations. The set of objects and relations is sometimes represented as a graph, connecting objects (nodes) with their relations (edges) and is known as a {\em scene graph} (Figure \ref{sg_example}). Scene graphs provide a compact representation of the semantics of an image, and can be useful for semantic-level interpretation and reasoning about a visual scene \cite{johnson2018image}. Scene-graph prediction is the problem of inferring the joint set of objects and their relations in a visual scene.

Since objects and relations are inter-dependent (e.g., a person and chair are more likely to be in relation ``sitting on'' than ``eating''), a scene graph predictor should capture this dependence in order to improve prediction accuracy. This goal is a special case of a more general problem, namely, inferring multiple inter-dependent labels, which is the research focus of the field of structured prediction. 
%Structured prediction addresses the problem of classification when the label space contains multiple inter-dependent labels. 
% For example, in semantic segmentation of an image, each pixel is assigned a label, while considering the labels of nearby pixels. A similar problem is the task of recognizing multiple entities and their relations in an image.
Structured prediction has attracted considerable attention because it applies to many learning problems and poses unique theoretical and algorithmic challenges \citep[e.g., see][]{belanger17a,chen2015learning,taskar03max}. It is therefore a natural approach for predicting scene graphs from images. 

\begin{figure}[t!]
	\begin{center}
        \includegraphics[width=\linewidth]{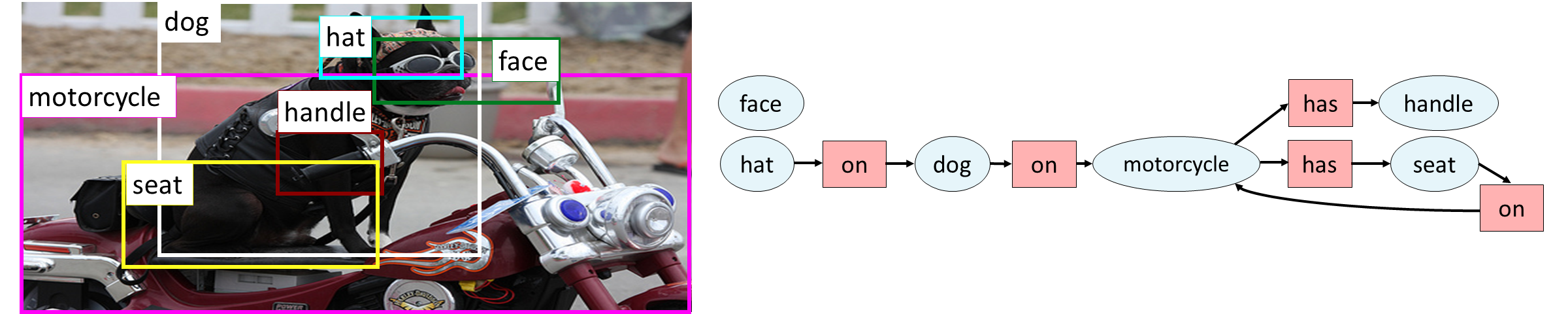}
	\caption{\small{
    An image and its scene graph from the Visual Genome dataset \citep{krishna2017visual}. The scene graph captures the entities in the image (nodes, blue circles) like \emph{dog} and their relations (edges, red circles) like \triplet{hat}{on}{dog}.}}
	\label{sg_example}
	\end{center}
\end{figure}

Structured prediction models typically define a score function $s(x,y)$ that quantifies how well a label assignment $y$ is compatible with an input $x$. In the case of understanding complex visual scenes, $x$ is an image, and $y$ is a complex label containing the labels of objects detected in an image and the labels of their relations. In this setup, the {\em inference task} amounts to finding the label that maximizes the compatibility score $y^* = \arg\max_y s(x,y)$. This score-based approach separates a scoring component -- implemented by a parametric model, from an optimization component -- aimed at finding a label that maximizes that score. Unfortunately, for a general scoring function $s(\cdot)$, the space of possible label assignments grows exponentially with input size. For instance, for scene graphs the set of possible object label assignments is too large even for relatively simple images, since the vocabulary of candidate objects may contain thousands of objects. As a result, inferring the label assignment that maximizes a scoring function is computationally hard in the general case. 

An alternative approach to score-based methods is to map an input $x$ to a structured output $y$ with a ``black box" neural network, without explicitly defining a score function. This raises a natural question: what is the right architecture for such a network? Here we take an axiomatic approach and argue that one important property such networks should satisfy is invariance to a particular type of input permutation. We then prove that this invariance is equivalent to imposing certain structural constraints on the architecture of the network, and describe architectures that satisfy these constraints. 
%, significantly extending the expressive power of current structured prediction approaches \jb{would the last clause "attract fire"?}. 
%jb: moving this to the empirical part since we have a synthetic data
%We argue \jb{and demonstrate empirically? (synthe)} that respecting permutation invariance is important, as otherwise the model would have to spend capacity on learning this invariance at training time.

%Conceptually, our approach is motivated by recent work on \textsc{DeepSets} \citep{deep_sets}, which asked a similar question for black-box functions on sets.

\ignore{After presenting our invariance assumptions and resulting structural characterization for neural architectures, we specify several models that have the corresponding structure. These indeed significantly extend the expressive power of current structured prediction approaches. Finally, we apply our approach to a recent challenging task of mapping an image to a  scene graph with both objects and relation labels. We evaluate on the {\em Visual Genome} dataset \citep{krishna2017visual}, and find that our model outperforms all current approaches, demonstrating the power of the new design principle in a challenging setup.}

To evaluate our approach, we first demonstrate on a synthetic dataset that respecting permutation invariance is important, because models that violate this invariance need more training data, despite having a comparable model size. Then, we tackle the problem of scene graph generation. We describe a model that satisfies the permutation invariance property, and show that it achieves state-of-the-art results on the competitive Visual Genome benchmark \citep{krishna2017visual}, 
demonstrating the power of our new design principle.
%\jb{should we write that we outperform things that are not GPI? I guess not because Yatskar is aslo GPI?}

In summary, the novel contributions of this paper are: a) Deriving sufficient and necessary conditions for  graph-permutation invariance in deep structured prediction architectures. b) Empirically demonstrating the benefit of graph-permutation invariance. c) Developing a state-of-the-art model for scene graph prediction on a large dataset of complex visual scenes.

%\ag{We should remember to discuss the Yatskar architecture, mentioning it is invariant, but relies on a specific order.}
% =============================
\section{Structured Prediction}
% =============================

Scored-based methods in structured prediction define a function $s(x,y)$ that quantifies the degree to which $y$ is compatible with $x$, and infer a label by
%solving  $y^*=\arg\max_{y} s(x,y)$ 
maximizing $s(x,y)$ \citep[e.g., see][]{belanger17a,chen2015learning,Lafferty01conditional,Meshi10,taskar03max}.
Most score functions previously used decompose as a sum over {\em simpler} functions, 
$s(x,y) = \sum_i f_i(x,y)$, making it possible to optimize $\max_y f_i(x,y)$  efficiently. This local maximization forms the basic building block of algorithms for approximately maximizing $s(x,y)$. One way to decompose the score function is to restrict each $f_i(x,y)$ to depend only on a small subset of the $y$ variables.
% \gal{edited a bit}

%\footnote{More precisely, many message passing algorithms require that functions $f_i(x,y) + \sum_k \delta_k(y_k)$ can be maximized efficiently.}

\ignore{
With the renewed interest in deep learning, there have been efforts to integrate deep networks with structured prediction. These included modeling the functions $f_i(x,y;\ww)$ using deep networks parametrized by $\ww$.
The most common score functions used are singleton $f(y_i,x)$ or pairwise $f_{ij}(y_i,y_j,x)$, and indeed several works modeled those using neural nets. Initial works in this direction used a two-stage architecture where local scores were learned independently of the structured prediction goal \citep[e.g., see][]{chen2014semantic,farabet2013learning,long2015fully}. Later works considered {\em end-to-end} architectures where the inference algorithm is part of the computation graph \citep{chen2015learning,PeiGC15,schwing2015fully,zheng2015conditional}.  The above works mostly used deep learning for representing the local score functions. The inference algorithms used were the standard approaches as in earlier structured prediction works. For example, loopy belief propagation (BP) and its variants \citep{chen2015learning}, mean field methods \citep{chen2015learning,schwing2015fully,stoyanov2011empirical} and gradient descent \citep{belanger17a,pmlr-v70-gygli17a}
}

The renewed interest in deep learning led to efforts to integrate deep networks with structured prediction, including modeling the $f_i$ functions as deep networks.
In this context, the most widely-used score functions are singleton $f_i(y_i,x)$ and pairwise $f_{ij}(y_i,y_j,x)$. The early work taking this approach used a two-stage architecture, learning the local scores independently of the structured prediction goal \citep{chen2014semantic,farabet2013learning}. Later studies considered {\em end-to-end} architectures where the inference algorithm is part of the computation graph \citep{chen2015learning,PeiGC15,schwing2015fully,zheng2015conditional}. 
%These studies used standard inference algorithms, such as loopy belief propagation, mean field methods and gradient descent \citep{belanger17a}.
Recent studies go beyond pairwise scores, also modelling global factors  \citep{belanger17a,pmlr-v70-gygli17a}.
 
Score-based methods provide several advantages. First, they allow intuitive specification of local dependencies between labels and how these translate to global dependencies. Second, for linear score functions, the learning problem has natural convex surrogates \cite{Lafferty01conditional,taskar03max}. Third, inference in large label spaces is sometimes possible via exact algorithms or empirically accurate approximations. 
\ignore{First, scores allow an intuitive specification of local dependencies between labels (e.g., pairwise dependencies) and how these translate to global dependencies. Second, when the  score function is linear in the model parameters (i.e., $s(x,y;\ww)$ is linear in $\ww$), the learning problem has natural convex surrogates (e.g., log-loss in CRF and structured hinge-loss in max margin Markov networks), and thus learning is efficient. Third, efficient inference in often possible either via exact algorithms or empirically accurate approximations. 
}
%\todoamir{Discuss that these advantages were preserved.}
However, with the advent of deep scoring functions $s(x,y;\ww)$, learning is no longer convex. Thus, it is worthwhile to rethink the architecture of structured prediction models, 
and consider models that map inputs $x$ to outputs $y$ directly without explicitly maximizing a score function. We would like these models to enjoy the expressivity and predictive power of neural networks, while maintaining the ability to specify local dependencies between labels in a flexible manner. In the next section, we present such an approach and consider a natural question: what should be the properties of a deep neural network used for structured prediction.
%\gal{The opening feels argumentative. How about starting with: "When designing deep structured prediction models, score-based methods have both benefits and limitations. As a first advantage, they ...  ."  The "However part" is also not very convincing. }

\section{Permutation-Invariant Structured Prediction}
% -----------------------------
%\subsection{Problem Definition}
% -----------------------------
In what follows we define the \textit{permutation-invariance} property for structured prediction models, and argue that permutation invariance is a natural principle for designing their architecture. 
%\roeih{Permutation invariance is a natural property for all structured prediction problems. It simply says that if the same input is presented to the model (up to permutation) then the same output should be generated (up to the same permutation). As we will show later in Sec. \ref{toy}, architectures which violate this property will tend overfit, as they ``waste'' capacity on generating invalid outputs. Specifically, all message passing algorithms (mean-field, belief propagation) are permutation invariant. Our approach significantly generalizes these, and clarify what is the key property required by such an algorithm.}

We first introduce our notation. We focus on structures with pairwise interactions, because they are simpler in terms of notation and are sufficient for describing the structure in many problems. We denote a structured label by $y=[y_1, \ldots, y_n]$. In a score-based approach, the score is defined via a set of singleton scores $f_i(y_i,x)$ and pairwise scores  $f_{ij}(y_i,y_j,x)$, where the overall score $s(x,y)$ is the sum of these scores. For brevity, we denote $f_{ij} = f_{ij}(y_i,y_j,x)$ and $f_i = f_i(y_i, x)$. An inference algorithm takes as input the local scores $f_i$, $f_{ij}$ and outputs an assignment that maximizes $s(x,y)$. We can thus view inference as a black-box that takes node-dependent and edge-dependent inputs (i.e., the scores $f_i$, $f_{ij}$) and returns a label $y$, even without an explicit score function $s(x,y)$. While numerous inference algorithms exist for this setup, including belief propagation (BP) and mean field, here we develop a framework for a deep labeling algorithm (we avoid the term ``inference'' since the algorithm does not explicitly maximize a score function). Such an algorithm will be a black-box, taking the $f$ functions as input and the labels $y_1,\ldots,y_n$ as output. We next ask what architecture such an algorithm should have. 

We follow with several definitions. A {\em graph labeling function} $\glf:(V,E)\rightarrow Y$ is a function whose input is an ordered set of node features $V=[\ve_1,\ldots,\ve_n]$ and an ordered set of edge features $E =[\ve_{1,2}\ldots,\ve_{i,j},\ldots,\ve_{n,n-1}]$. For example, $\ve_i$ can be the array of values $f_i$, and $\ve_{i,j}$ can be the table of values $f_{i,j}$. Assume $\ve_i\in\reals^d$ and $\ve_{i,j}\in\reals^e$.  The output of $\glf$ is a set of node labels $\yy = [y_1,\ldots,y_n]$. Thus, algorithms such as BP are graph labeling functions. However, graph labeling functions do not necessarily maximize a score function. We denote the joint set of node features and edge features by $\ve$ (i.e., a set of $n + n(n-1)=n^2$ vectors).
In \secref{sec:general} we discuss extensions to this case where only a subset of the edges is available.
%, since they take $f$ as input and output labels

% \begin{figure}[t]
% 	\vskip 0.2in
% 	\begin{center}
% 	\centerline{\includegraphics[width=0.8\columnwidth]{Figures/invariant_graph_zipped_4}}
% 	\caption{{\bf {Graph permutation invariance and structured prediction}}. A graph labeling function $\glf$ is  graph permutation invariant (GPI) if permuting the names of nodes maintains the output.}
% 	\label{fig:invariant_graph}
% 	\end{center}
% 	\vskip -0.2in
% \end{figure}

% Combined figures
\begin{figure}[t]
    \begin{center}
        \includegraphics[width=\linewidth]{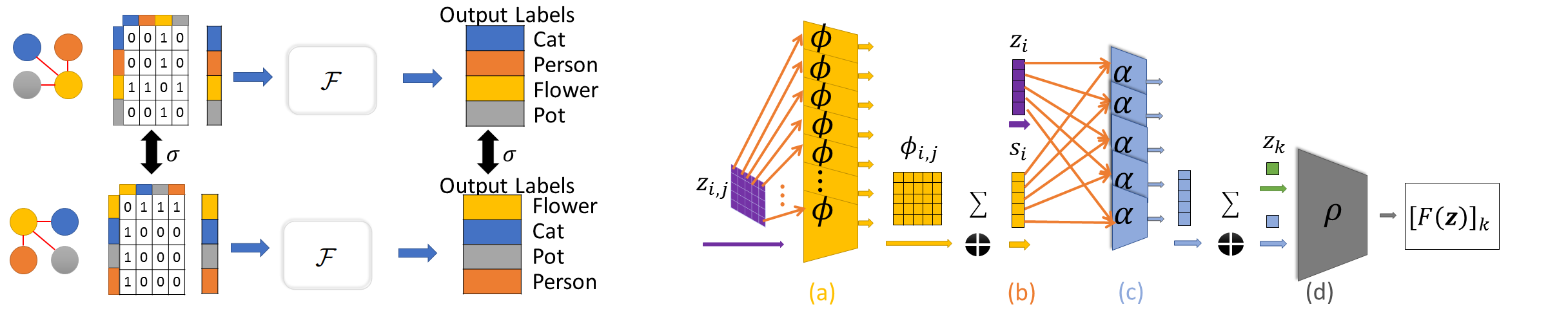}
    \caption{{\bf Left}: Graph permutation invariance. A graph labeling function $\glf$ is  graph permutation invariant (GPI) if permuting the node features maintains the output.
    {\bf Right}: a schematic representation of the GPI architecture in \hyperref[graph_permutation_form] {Theorem \ref{thm:gpi}}. Singleton features $\ve_i$ are omitted for simplicity. \textbf{(a)} First, the features $\ve_{i,j}$ are processed element-wise by $\phi$.  \textbf{(b)} Features are summed to create a vector $\ss_i$, which is concatenated with $\ve_i$.  \textbf{(c)} A representation of the entire graph is created by applying $\alphav$ $n$ times and summing the created vector. \textbf{(d)} The graph representation is then finally  processed by $\rho$ together with $\ve_k$. }
    \label{fig:f_gpi}
    \end{center}
\end{figure}

A natural requirement is that the function $\cF$ produces the same result when given the same features, up to a permutation of the input. For example, consider a label space with three variables $y_1, y_2, y_3$, and assume that $\cF$ takes as input 
$\ve = (\ve_1, \ve_2, \ve_3, \ve_{12},\ve_{13},\ve_{23}) = (f_1, f_2, f_3, f_{12}, f_{13}, f_{23})$, and outputs a label $\yy = (y^*_1, y^*_2, y^*_3)$.  When $\cF$ is given an input that is permuted in a consistent way, say, $\ve' =  (f_2, f_1, f_3, f_{21}, f_{23}, f_{13})$, this defines {\em exactly} the same input. Hence, the output should still be $\yy = (y^*_2, y^*_1, y^*_3)$.  Most inference algorithms, including BP and mean field, satisfy this symmetry requirement by design, but this property is not guaranteed in general in a deep model. Here, our goal is to design a deep learning black-box, and hence we wish to guarantee invariance to input permutations. A black-box that violates this invariance ``wastes'' capacity on learning it at training time, which increases sample complexity, as shown in Sec. \ref{toy}. We proceed to formally define the permutation invariance property.
%We next consider what happens to a graph labeling function, when node features are permuted by a permutation $\sigma$. Importantly, edge features are also permuted in a way that is consistent with $\sigma$.

\begin{definition} \label{def:permutation}
Let $\ve$ be a set of node features and edge features, and let $\sigma$ be a permutation of $\{1,\ldots,n\}$. We define $\sigma(\ve)$ to be a new set of node and edge features given by $[\sigma(\ve)]_i= \ve_{\sigma(i)}$ and $[\sigma(\ve)]_{i,j} = \ve_{\sigma(i),\sigma(j)}$.
%\begin{equation}
%[\sigma(\ve)]_i= \ve_{\sigma(i)} \ \ \ \   , \ \ \ \ %[\sigma(\ve)]_{i,j} = \ve_{\sigma(i),\sigma(j)}.
%\end{equation}
\end{definition}
%$\sigma(\ve)$ has the same elements as in $\ve$, but node features are permuted according to $\sigma$, and edge elements are permuted accordingly. 
We also use the notation $\sigma([y_1,\ldots,y_n]) = [y_{\sigma(1)}, \ldots, y_{\sigma(n)}]$ for permuting the labels. Namely, $\sigma$ applied to a set of labels yields the same labels, only permuted by $\sigma$. Be aware that applying $\sigma$ to the input features  is different from permuting labels, because edge input features must permuted in a way that is consistent with permuting node input features. We now provide our key definition of a function whose output is invariant to permutations of the input. See \figref{fig:f_gpi} (left).
\begin{definition}\label{graph_permutation_invariant_definition}
A graph labeling function $\glf$ is said to be {\bf graph-permutation invariant} (GPI), if for all permutations $\sigma$ of $\{1,\ldots,n\}$ and for all $\ve$ it satisfies:
%\begin{equation}
$ \glf(\sigma(\ve)) = \sigma(\glf(\ve))$.
%\end{equation}
\end{definition}
%\vspace{-10pt}

% \figref{fig:invariant_graph} 
%The above property says that as long as the input to $\glf$ describes the same node and edge properties, $\glf$ returns the same labeling. This is indeed a property we expect from any $\glf$, and we thus turn to characterizing a necessary and sufficient structure for achieving it.

% -----------------------------
\subsection{Characterizing Permutation Invariance}
% -----------------------------
Motivated by the above discussion, we ask: what structure is necessary and sufficient to guarantee that $\glf$ is GPI?
Note that a function $\glf$ takes as input an {\bf ordered} set $\ve$. Therefore its output on $\ve$ could certainly differ from its output on $\sigma(\ve)$. To achieve permutation invariance, $\glf$ should  contain certain symmetries. For instance, one permutation invariant architecture could be to define $y_i = g(\ve_i)$ for any function $g$, but this architecture is too restrictive and does not cover all permutation invariant functions. Theorem \ref{thm:gpi} below provides a complete characterization (see \figref{fig:f_gpi} for the corresponding architecture). Intuitively, the architecture in Theorem \ref{thm:gpi} is such that it can aggregate information from the entire graph, and do so in a permutation invariant manner.
%\roeih{But first, let us explain why the theorem is intuitively correct. The idea is that one can define functions $\alpha,\phi$ such that the input to the $\rho$ function is the set of features $\zz$ of the entire graph. In other words, the functions $\alpha$ and $\phi$ aggregate information from the graph in a permutation invariant manner.}
% \figref{high_level_arch}
\begin{theorem}
\label{thm:gpi}
Let $\glf$ be a graph labeling function. Then $\glf$ is graph-permutation invariant if and only if there exist functions $\alphav,\rho,\phiv$ such that for all $k=1,\ldots,n$: 
\begin{equation}
    \label{graph_permutation_form}
    [\glf(\ve)]_k = \rho\left(\ve_k,\sum_{i=1}^n \boldsymbol{\alpha} \left(\ve_i, \sum_{j \neq i} \boldsymbol{\phi}(\ve_{i}, \ve_{i,j}, \ve_{j})\right)\right), 
\end{equation}
where $\boldsymbol{\phi}:\reals^{2d+e} \rightarrow \reals^{L}$, $\boldsymbol{\alpha}:\reals^{d+L} \rightarrow \reals^W$ and $\boldsymbol{\rho}:\reals^{W+d}\to \reals$. 
%\jb{what about mult-class outputs? the output here is one real number.} Here $[\cdot]_k$ denotes the $k^{\text{th}}$ element \jb{this was already defined}. \\ 
%\todoamir{There is something confusing if the domains of the different $y_k$ are different. It's resolved by the uniqueness of $z_k$ but still confusing. Need to think how to clarify this.}
%\todoamir{I removed the $N(i)$ notation since we did not discuss incomplete graphs. We should consider those at some point in the text.}
\end{theorem}
\begin{proof}

First, we show that any $\glf$ satisfying the conditions of Theorem \ref{thm:gpi} is GPI. Namely, for any permutation $\sigma$, $[\glf(\sigma(\ve))]_k = [\glf(\ve)]_{\sigma(k)}$. To see this, write $[\glf(\sigma(\ve))]_k$ using \eqref{graph_permutation_form} and Definition~\ref{def:permutation}:
\begin{equation}\label{graph_permutation_proof}
    [\glf(\sigma(\ve))]_k = \rho (\ve_{\sigma(k)},\sum_{i} \boldsymbol{\alpha} (\ve_{\sigma(i)}, \sum_{j\neq i} \boldsymbol{\phi}(\ve_{\sigma(i)}, \ve_{\sigma(i),\sigma(j)}, \ve_{\sigma(j)}))).
\end{equation}
% $\rho(\ve_{\sigma(k)},\sum_{i} \boldsymbol{\alpha} (\ve_{\sigma(i)}, \sum_{j\neq i} \boldsymbol{\phi}(\ve_{\sigma(i)}, \ve_{\sigma(i),\sigma(j)}, \ve_{\sigma(j)}))).$
The second argument of $\rho$ above is invariant under $\sigma$, because it is a sum over nodes and their neighbors, which is invariant under permutation. Thus \eqref{graph_permutation_proof} is equal to:
\begin{align*}
\rho(\ve_{\sigma(k)},\sum_{i} \boldsymbol{\alpha} (\ve_{i}, \sum_{j\neq i} \boldsymbol{\phi}(\ve_{i}, \ve_{i,j}, \ve_{j}))) &= [\glf(\ve)]_{\sigma(k)}
\end{align*} where equality follows from \eqref{graph_permutation_form}. We thus proved that \eqref{graph_permutation_form} implies graph permutation invariance.
% $\rho(\ve_{\sigma(k)},\sum_{i} \boldsymbol{\alpha} (\ve_{i}, \sum_{j\neq i} \boldsymbol{\phi}(\ve_{i}, \ve_{i,j}, \ve_{j}))) = [\glf(\ve)]_{\sigma(k)}$ 

Next, we prove that {\em any} given GPI function $\glf_0$ can be expressed as a function $\glf$ in \eqref{graph_permutation_form}.
Namely, we  show how to define  $\boldsymbol{\phi},\boldsymbol{\alpha}$ and $\rho$ that can implement $\glf_0$. Note that in this direction of the proof the function $\glf_0$ is a black-box. Namely, we only know that it is GPI, but do not assume anything else about its implementation.

The key idea is to construct $\phiv,\alphav$ such that the second argument of $\rho$ in \eqref{graph_permutation_form} contains  the information about {\em all} the graph features $\ve$. Then, the function $\rho$ corresponds to an application of $\glf_0$ to this representation, followed by extracting the label $y_k$. To simplify notation assume edge features are scalar ($e=1$). The extension to vectors is simple, but involves more indexing. 

We assume WLOG that the black-box function $\glf_0$ is a function only of the pairwise features $\ve_{i,j}$ (otherwise, we can always augment the pairwise features with the singleton features). Since $\ve_{i,j}\in\reals$ we use a matrix $\reals^{n,n}$ to denote all the pairwise features.

Finally, we assume that our implementation of $\glf_0$ will take additional node features $z_k$ such that no two nodes have the same feature (i.e., the features identify the node). 

Our goal is thus to show that there exist functions $\alpha,\phiv,\rho$ such that the function in \eqref{graph_permutation_proof} applied to $Z$ yields the same labels as $\glf_0(Z)$.

%\roeih{Meaning, we assume that the given GPI function $\mathcal{F}$ does not use singleton features explicitly, but their information is contained in the pair features, such that the node features $z_k$ are used for defining the functions $\rho$, $\alpha$, $\phi$, but $\mathcal{F}$ itself does not use $z_k$ as input.}

% \begin{figure}[t]
%     \begin{center}
%  \centerline{\includegraphics[width=0.6\columnwidth]{Figures/high_level_arch_final_fixed}}
%     \caption{A schematic representation of the GPI architecture in \hyperref[graph_permutation_form] {Theorem \ref{thm:gpi}}. Singleton features $\ve_i$ are omitted for simplicity.  First, the features $\ve_{i,j}$ are processed element-wise by $\phi$.  Next, they are summed to create a vector $\ss_i$, which is concatenated with $\ve_i$.  Third, a representation of the entire graph is created by applying $\alphav$ $n$ times and summing the created vector. The graph representation is then finally  processed by $\rho$ together with $\ve_k$. }
%     \label{high_level_arch}
%     \end{center}
% \end{figure}

Let $H$ be a hash function with $L$ buckets mapping node features $\ve_i$ to an index (bucket). Assume that $H$ is perfect (this can be achieved for a large enough $L$). Define $\boldsymbol{\phi}$ to map the pairwise features to a vector of size $L$. Let $\onehotvector{j}$ be a one-hot vector of dimension $\reals^L$, with one in the $j^{\text{th}}$ coordinate. Recall that we  consider scalar $\ve_{i,j}$ so that $\phiv$ is indeed in $\reals^L$, and define $\phiv$ as: $\phiv(\ve_i, \ve_{i,j}, \ve_j)=\onehotvector{H(\ve_j)} \vescalar_{i,j}$, 
%\begin{equation}
%	\phiv(\ve_i, \ve_{i,j}, \ve_j)=\onehotvector{H(\ve_j)} \vescalar_{i,j}
%\end{equation}
i.e., $\phiv$ ``stores'' $\vescalar_{i,j}$ in the unique bucket for node $j$. 

Let $\ss_i=\sum_{\ve_{i,j} \in E} \boldsymbol{\phi}(\ve_{i}, \vescalar_{i,j}, \ve_{j})$ be the second
argument of $\alphav$ in \eqref{graph_permutation_form} ($\ss_i\in\reals^L$). Then, since all $\ve_{j}$ are distinct, $\ss_i$ stores all the pairwise features for neighbors of $i$ in unique positions within its $L$ coordinates. Since  $\ss_i(H(\ve_k))$ contains the feature $z_{i,k}$ whereas $\ss_j(H(\ve_k))$ contains the feature $\ve_{j,k}$, we cannot simply sum the $\ss_i$, since we would lose the information of which edges the features originated from. Instead, we define $\alphav$ to map $\ss_i$ to $\reals^{L\times L}$ such that each feature is mapped to a distinct location.
%\footnote{In the notation of the theorem, we thus have $W=L^2$.} 
Formally:
\begin{equation}
\alphav(\zz_i, \ss_i) = \onehotvector{H(\zz_i)}\ss_i^T ~.
\label{eq:alpha}
\end{equation}
$\alphav$ outputs a matrix that is all zeros except for the features corresponding to node $i$ that are stored in row $H(\zz_i)$. The matrix $M = \sum_i \alphav(\zz_i, \ss_i)$ (namely, the second argument of $\rho$ in \eqref{graph_permutation_form}) is a matrix with all the edge features in the graph including the graph structure.

To complete the construction we set $\rho$ to have the same outcome as $\glf_0$. We first discard rows and columns in $M$ that do not correspond to original nodes (reducing $M$ to dimension $n\times n$).  Then, we use the reduced matrix as the input $\ve$ to the black-box $\glf_0$. 

Assume for simplicity that $M$ does not need to be contracted (this merely introduces another indexing step). Then $M$  corresponds to the original matrix $Z$ of pairwise features, with both rows and columns permuted according to $H$.
%Let the output of $\glf_0$ on $M$ be $\yy=[y_1,\ldots,y_n$].
We will thus use $M$ as input to the function $\glf_0$. Since $\glf_0$ is GPI, this means that the label for node $k$ will be given by $\glf_0(M)$ in position $H(\zz_k)$. Thus we set $\rho(\zz_k,M) = [\glf_0(M)]_{H(\zz_k)}$, and by the argument
above this equals $[\glf_0(Z)]_k$, implying that the above $\alphav,\phiv$ and $\rho$ indeed implement $\glf_0$.
\end{proof}
\vspace{-0.6cm}
\paragraph{Extension to general graphs}
\label{sec:general}
So far, we discussed complete graphs, where edges correspond to valid feature pairs. However, many graphs of interest might be incomplete. For example, an $n$-variable chain graph in sequence labeling has only $n - 1$ edges. For such graphs, the input to $\glf$ would not contain all $\ve_{i,j}$ pairs but rather only features corresponding to valid edges of the graph, and we are only interested in invariances that preserve the graph structure, namely, the automorphisms of the graph. Thus, the desired invariance is that $\sigma(\glf(\ve))=\glf(\sigma(\ve))$, where $\sigma$ is not an arbitrary permutation but an automorphism. It is easy to see that a simple variant of \hyperref[graph_permutation_form] {Theorem 1} holds in this case. All we need to do is 
replace in \eqref{graph_permutation_proof} the sum $\sum_{j\neq i}$ with $\sum_{j\in N(i)}$, where $N(i)$ are the neighbors of node i in the graph. The arguments are then similar to the proof above. 
%First, any $\glf$ satisfying the conditions of \hyperref[graph_permutation_form] {Theorem 1} will be invariant under permutations that preserve the structure of the graphs.
%To see that any such architecture is GPU, notice that the same elements will be included in the sum operations, just under permutation. Next, any black-box GPI function can be expressed as in \eqref{graph_permutation_form} under permutations. In this case, we will also replace $\sum_{j\neq i}$ with $\sum_{j\in N(i)}$ and notice that all the input features will still be encoded in the invariant representation of the graph.      

\paragraph{Implications of Theorem 1}
%\label{sec:DeepGraphPrediction}
%\subsection{Architecture}
\ignore{
The architecture is required to be both expressive and invariant to graph permutation as defined in (\ref{graph_permutation_invariant_definition}).
Naturally, those requirements will be fulfilled by an architecture which based on \hyperref[graph_permutation_form] {Theorem 1}.

\jb{It is unclear what is the point of this section. it seems like a discussion on the last section but no clear unifying thread...}
\hyperref[graph_permutation_form] {Theorem \ref{thm:gpi}} suggest that the inferring process should be composed from two steps.
First capture the essence of  the graph following by a step which makes a prediction using the captured features together with the features of the specified node.
}

%\hyperref[graph_permutation_form] {Theorem (\ref{thm:gpi})} provides the general requirements for designing an architecture for structured prediction. For a given problem, one has to choose a specific architecture and parameterization for $\alphav$, $\phiv$, $\rho$.
%\hyperref[graph_permutation_form] {Theorem (\ref{thm:gpi})}
Our result has interesting implications for deep structured prediction.
First, it highlights that the fact that the architecture ``collects'' information from {\em all} different edges of the graph, in an invariant fashion via the $\alphav,\phiv$ functions. Specifically, the functions $\phi$ (after summation) aggregate all the features around a given node, and then $\alpha$ (after summation) can collect them. Thus, these functions can provide a summary of the entire graph that is sufficient for downstream algorithms. This is different from one round of message passing algorithms which would not be sufficient for collecting global graph information. Note that the dimensions of $\phi,\alpha$ may need to be large to aggregate all graph information (e.g., by hashing all the features as in the proof of Theorem \ref{thm:gpi}), but the architecture itself can be shallow. 
%Unlike message passing algorithms, this is a ``shallow'' architecture, and suggests inference algorithms may not require very deep models. 
%

Second, the architecture is parallelizable, as  all $\phiv$ functions can be applied simultaneously. This is in contrast to recurrent models \cite{neural_motifs} which are harder to parallelize and are thus slower in practice.
\ignore{
\jb{I would delete the next paragraph, not clear what you learn here.}
Third, the theorem uses $\ve_i$ in both $\alphav$ and $\rho$, which in the theorem we used to ``identify'' a node with respect to the description of the graph. This is not a commonly used idea in many neural message passing schemes, and we find that it both serves the proof, and is useful in practice. \jb{compared to what? I don't think we show it's useful in practice}
}

Finally, the theorem suggests several common architectural structures that can be used within GPI. We briefly mention two of these.
% \hyperref[graph_permutation_form] {Theorem (\ref{thm:gpi})} provides the general requirements for designing an architecture that is GPI. For a given problem, one has to instantiate and parameterize $\alphav$, $\phiv$, $\rho$.
%JB: commented this out to save space and also thought it is a bit meh.
%E.g., it is interesting to consider how an algorithm like BP can be implemented in our framework. Following the proof of Theorem \ref{thm:gpi}, one would use $\phiv,\alphav$ to aggregate features, and then $\rho$ would apply BP to these features. Our architecture is of course more general by construction. For example, it could use $\phiv$ and $\alphav$ to ``sketch'' the input graph, such that labeling can be performed on a reduced representation.
%For example, belief propagation can be implemented using this architecture by having the message passing itself implemented within the $\rho$ function. \gc{Amir, analyzing how BP can be implemented could be insightful could you add details?}.
%We now survey certain architectures consistent with Theorem 1 and discuss their expressive power.
%Opposed to belief propagation which perform the inference by multiple steps of message passing, \hyperref[graph_permutation_form] {Theorem 1} also implies that two stages allows us to capture the essence of the whole graph. However, the prediction step, which its input is a complexed graph representation, is expected to be harder than a single step in belief propagation. 
%\vspace{-10pt}
%\label{attention}
1) {\bf Attention:}
% -------------------------------------------
Attention is a powerful component in deep learning architectures \citep{bahdanau2015neural}, but most inference algorithms do not use attention. 
%We now show how attention can be introduced in our framework. 
Intuitively, in attention each node $i$ aggregates features of neighbors through a weighted sum, where the weight is a function of the neighbor's relevance. 
%Intuitively, attention means that instead of aggregating features of neighbors, a node $i$ weighs neighbors based on their relevance. 
For example, the label of an entity in an image may depend more strongly on entities that are spatially closer.
Attention can be naturally implemented in our GPI characterization, and we provide a full derivation for this implementation in the appendix. It plays a key role in our scene graph model described below.
%\label{rnn}
2) {\bf RNNs:}
% ----------------------------------------------
%\hyperref[graph_permutation_form] {Theorem 1} %allows arbitrary functions for $\phiv$, $\alphav$ and $\rho$.
%jb: delete to save space
%except for their input dimensionality. 
%Specifically, these functions can involve recursive computation, simulate existing message passing algorithms, and new algorithms that are learned from data. 
%JB: this is not true given the next paragraph.
%Indeed, in our application of scene graph generation (section \ref{sg_model}), we used RNNs as part of our architecture, and found this to improve performance. 
%jb: delete to save space
%This can of course be extended to more elaborate structures like LSTMs \citep{hochreiter1997lstm} and Neural Turing Machines \citep{graves2014neural}, which we leave for future work.
Because GPI functions are closed under composition, for any GPI function $\glf$ 
%\hyperref[graph_permutation_form] {Theorem 1} suggests that any function in the form of $\glf$ is graph permutation invariant. It is easy to show that composing two functions that  are GPI, is also GPI.
%Therefore, 
we can run $\glf$ iteratively
by providing the output of one step of $\glf$ as part of the input to the next step and maintain GPI. This results in a recurrent architecture, which we use in our scene graph model.

\section{Related Work}

The concept of architectural invariance was recently proposed in \textsc{DeepSets} \citep{deep_sets}. The invariance we consider is much less restrictive: the architecture does not need to be invariant to all permutations of singleton and pairwise features, just those consistent with a graph re-labeling. This characterization results in a substantially different set of possible architectures.

{\bf Deep structured prediction}. There has been significant recent interest in extending deep learning to structured prediction tasks. Much of this work has been on semantic segmentation, where convolutional networks \citep{fcn} became a standard approach for obtaining ``singleton scores'' and various approaches were proposed for adding structure on top. Most of these approaches used variants of message passing algorithms, unrolled into a computation graph \citep{sg_generation_msg_pass}. Some studies parameterized parts of the message passing algorithm and learned its parameters \citep{lin2015deeply}. Recently, gradient descent has also been used for maximizing score functions \citep{belanger17a,pmlr-v70-gygli17a}. An alternative to deep structured prediction is greedy decoding, inferring each label at a time based on previous labels. This approach has been  popular in sequence-based applications (e.g., parsing \citep{chen2014fast}), relying on the sequential structure of the input, where BiLSTMs are effectively applied. Another related line of work is applying deep learning to graph-based problems, such as TSP \citep{bello2016neural,gilmer2017neural,khalil2017learning}. Clearly, the notion of graph invariance is important in these, as highlighted in \citep{gilmer2017neural}. They however do not specify a general architecture that satisfies invariance as we do here, and in fact focus on message passing architectures, which we strictly generalize. Furthermore, our focus is on the more general problem of structured prediction, rather than specific graph-based optimization problems.

\textbf{Scene graph prediction. } Extracting scene graphs from images provides a semantic representation that can later be used for reasoning, question answering, and image retrieval \citep{img_retriev_using_sg, lang_prior, entangled_scene}. It is  at the forefront of machine vision research, integrating challenges like object detection, action recognition and detection of human-object interactions  \citep{support_relations, plummerPLCLC2017}. Prior work on scene graph predictions used neural message passing algorithms \citep{sg_generation_msg_pass} as well as prior knowledge in the form of word embeddings \citep{lang_prior}. Other work suggested to predict graphs directly from pixels in an end-to-end manner \cite{pixels_to_graph}. \text{NeuralMotif} \citep{neural_motifs}, currently the state-of-the-art model for scene graph prediction on Visual Genome, employs an RNN that provides global context by sequentially reading the independent predictions for each entity and relation and then refines those predictions. The \textsc{NeuralMotif} model maintains GPI by fixing the order in which the RNN reads its inputs and thus only a single order is allowed. However, this fixed order is not guaranteed to be optimal. %Conversely, our architecture is guaranteed to provide the same output for all valid input permutations.

\ignore{
\begin{itemize}
    \item{Deep Sets \citep{deep_sets}.}
    \item{Fei-Fei 2 papers: Scene Graph Generation by Iterative Message Passing and Visual Relationship Detection with Language Priors. \citep{sg_generation_msg_pass,lang_prior}}
    \item{Pixels to Graphs by Associative Embedding \citep{pixel2graph}.}
    \item{Neural Motifs: Scene Graph Parsing with Global Context \citet{neural_motifs}.}
    \item{Relation Networks by DeepMind \citep{relation_network}.}
    \item{The More You Know: Using Knowledge Graphs for Image Classification by Ruslan Salakhutdinov \citep{knowledge_graphs}.}
    \item{Neural message passing (by Oriol and recent NIPS paper by Le Song) \citep{dl_msg_pass_inf}}
\end{itemize}
}

\section{Experimental Evaluation}
\label{sec:experiments}
We empirically evaluate the benefit of GPI architectures. First, using  a synthetic graph-labeling task, and then for the problem of mapping images to scene graphs.

\subsection{Synthetic Graph Labeling}
\label{toy}
% =======================================================
We start with studying GPI on a synthetic problem, defined as follows. An input graph $G=(V,E)$ is given, where each node $i \in V$ is assigned to one of $K$ sets. The set for node $i$ is denoted by $\Gamma(i)$. The goal is to compute for each node the number of neighbors that belong to the same set. Namely, the label of a node is $y_i=\sum_{j \in N(i)} \mathbb{1}[\Gamma(i) = \Gamma(j)]$. We generated random graphs with 10 nodes (larger graphs produced similar results) by sampling each  edge independently and uniformly, and sampling  $\Gamma(i)$ for every node uniformly from $\{1,\ldots,K\}$.
The node features $\ve_i \in \{0,1\}^K$ are one-hot vectors of $\Gamma(i)$ and the edge features $\ve_{i,j}\in\{0,1\}$ indicate whether $ij\in E$. We compare two standard non-GPI architectures and one GPI architecture: 
(a) A GPI-architecture for graph prediction, described in detail in Section \ref{SG}. We used the basic version without attention and RNN. 
(b) LSTM: We replace $\sum\phi(\cdot)$ and $\sum\alpha(\cdot)$, which perform aggregation in Theorem \ref{thm:gpi} with two LSTMs with a state size of 200 that read their input in random order. 
(c) A fully-connected (FC) feed-forward network with 2 hidden layers of 1000 nodes each.
The input to the fully connected model is a concatenation of all node and pairwise features. The output is all node predictions. The focus of the experiment is to study sample complexity. Therefore, for a fair comparison, we use the same number of parameters for all models.

\begin{figure}	
\begin{tabular}{ll}
    \includegraphics[scale=0.3]{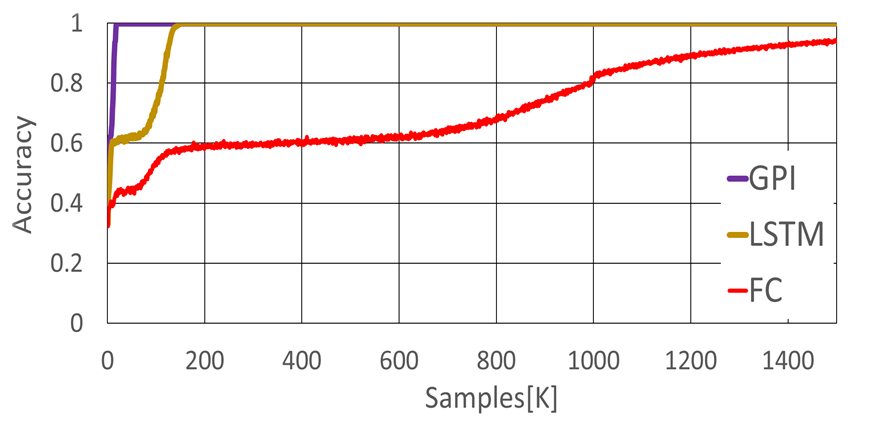}
%&
    \includegraphics[scale=0.3]{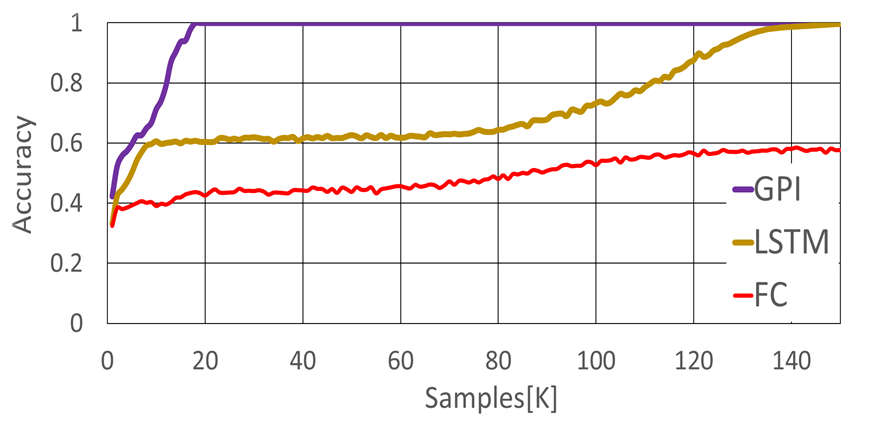}
\end{tabular}
% \caption{Left: Accuracy as a function of \#samples. We can see that SGP model converge with much less number of examples. Right: Accuracy as a function Time in seconds. The SGP model get 1.0 accuracy 12 times faster than the LSTM model.}
\caption{Accuracy as a function of sample size for graph labeling. Right is a zoomed in version of left.}
\label{fig:toy_results}
\end{figure}

\figref{fig:toy_results}, shows the results, demonstrating that GPI requires far fewer samples to converge to the correct solution.
%. The left panel traces accuracy as a function of number of samples, showing that  inference is solved by the GPI architectures already after 25K samples. In comparison, the LSTM needs 125K samples for convergence, and the FC network fails to solve the problem, because it requires many more parameters and samples compared to the two other models.
%The inference problem of graphs with 10 nodes solved by the FC model after 1500K samples.
This illustrates the advantage of an architecture with the correct inductive bias for the problem. 
%Moreover, as shown on the right panel, the GPI architecture is parallelizable, compared to the sequential  LSTM, leading to 12x faster training .

\vspace{-5pt}
\subsection{Scene-Graph Classification}\label{SG}
% ==========================================================
We evaluate the GPI approach on the motivating task of this paper, inferring scene graphs from images (Figure \ref{sg_example}). In this problem, the input is an image annotated with a set of {\em bounding boxes} for the entities in the image.\footnote{For simplicity, we focus on the task where boxes are given.} The goal is to label each bounding box with the correct entity category and every pair of entities with their relation, such that they form a coherent {\em scene graph}. 
%Relations between entities can be spatial (\emph{``on"}) or functional (\emph{``wearing"}). 
%As an example, an image with 5 bounding boxes has $5 + 5 \times 4 = 25$ output variables.

%This concept is illustrated in \figref{sg_example}, showing an image of a dog on a motorcycle (left) and the corresponding scene graph (right).  The pink box in the image is labeled \emph{``motorcycle"} and the white box is labeled \emph{``dog"}. These two boxes correspond to two nodes (light blue circles in \figref{sg_example} right), and their relation \emph{``on"} corresponds to an edge (red circle labeled \emph{``on''}). While scene graphs are typically sparse, one can view a scene graph as complete if each pair of unrelated entities is connected by a `null' edge. A scene graph can be represented as a collection of triplets, each with a relation and two entities, like \triplet{\emph{dog}}{\emph{on}}{\emph{motorcycle}}.

% \newcommand{\lp}{Label Predictor}
% \subsubsection{Model}\label{sg_model}
% Our model has two components: A {\em \lp{} (LP)} that takes as input an image with bounding boxes and outputs a distribution over labels for each entity and relation. Then, a {\em Scene Graph Predictor (SGP)} that takes all label distributions and predicts more consistent label distributions jointly for all entities and relations.

%\subsubsection{Model}\label{sg_model}
%\gal{This section is not for mortal beings.}
%\ag{Below, fill in values for R,L}
We begin by describing our {\em Scene Graph Predictor (SGP)} model.  We aim to predict two types of variables. The first is entity variables $[y_1, \dots, y_n]$ for all bounding boxes. Each $y_i$ can take one of $L$ values (e.g., ``dog'', ``man''). The second is relation variables $[y_{n+1}, \dots, y_{n^2}]$ for every pair of bounding boxes. Each such $y_j$ can take one of $R$ values (e.g., ``on'', ``near''). Our graph connects variables that are expected to be inter-related. It contains two types of edges: 1) {\bf entity-entity edge} connecting every two entity variables $(y_i$ and $y_j$ for $1 \leq i \neq j \leq n$. 2) {\bf entity-relation edges} connecting every relation variable $y_k$ (where $k>n$) to its two entity variables. 
%In the graph we consider the nodes are $[y_1, \dots, y_{n^2}]$. There are edges between all entity variables $(y_i, y_j)$ for $1 \leq i \neq j \leq n$, as well as edges that connect a relation variable $y_r$, where $n < r \leq n^2$ with the two entity variables it describes. 
Thus, our graph is not a complete graph and our goal is to design an architecture that will be invariant to any automorphism of the graph, such as permutations of the entity variables.
%\gal{This last paragraph needs work. we should clarify the relation between edges in this graph and $z_{ij}$ }

For the input features $\ve$, we used the features learned by the baseline model from \cite{neural_motifs}.\footnote{The baseline does not use any LSTM or context, and is thus unrelated to the main contribution of \cite{neural_motifs}.} Specifically, the entity features $\ve_i$ included (1) The confidence probabilities of all entities for $y_i$ as learned by the baseline model. (2) Bounding box information given as \texttt{(left, bottom, width, height)}; (3) The number of smaller entities (also bigger); (4) The number of entities to the left, right, above and below. (5) The number of entities with higher and with lower confidence; (6) For the linguistic model only: word embedding of the most probable class. Word vectors were learned with \textsc{GloVe} from the ground-truth captions of Visual Genome.

Similarly, the relation features $\ve_j\in\reals^R$ contained the probabilities of relation entities for the relation $j$. For the Linguistic model, these features were extended to include word embedding of the most probable class. For entity-entity pairwise features $\ve_{i,j}$, we use the relation probability for each pair. Because the output of SGP are probability distributions over entities and relations, we use them as an the input $\ve$ to SGP, once again in a recurrent manner and maintain GPI. %Finally, we also explored the use of word embeddings. For each $\zz_i$ above, we concatenate the embedding of the most probable entity and relation label. 

% The $\rho$ predictions were then again used as $\ve$ and the process repeated one more time (i.e., two step RNN). 

%\footnote{Word vectors were learned with \textsc{GloVe} \citep{pennington2014glove} from the ground-truth captions of Visual Genome \citep{krishna2017visual}.}

%We define the input features $\ve$ as follows.  For $1 \leq i \leq n$, $\ve_i$ is the concatenation of $\ve_i^{\text{features}}$, which is a distribution over the entity categories. 

%\gal{DEFINE  $\ve_i^{\text{features}}$ before using.}

%Our model has {\em Scene Graph Predictor (SGP)} component that takes all label distributions and predicts more consistent label distributions jointly for all entities and relations. We used \textsc{\citep{neural_motifs}} distributions as an initial baseline input for the SGP model.

\newcommand{\sE}{\mathcal{P}}
\newcommand{\sR}{\mathcal{R}}

We next describe the main components of the GPI architecture.
First, we focus on the parts that output the entity labels. 
$\boldsymbol{\phi}_{ent}$ is the network that integrates features for two entity variables $y_i$ and $y_j$. It simply takes $\ve_i$, $\ve_j$ and $\ve_{i,j}$ as input, and outputs a vector of dimension $n_1$. Next, the network $\alphav_{ent}$ takes as input the outputs of $\boldsymbol{\phi}_{ent}$ for all neighbors of an entity, and uses the attention mechanism described above to output a vector of dimension $n_2$. Finally, the $\rho_{ent}$ network takes these $n_2$ dimensional vectors and outputs $L$ logits predicting the entity value. The $\rho_{rel}$ network takes as input the $\alpha_{ent}$ representation of the two entities, as well as $\ve_{i,j}$ and transforms the output into $R$ logits. See appendix for specific network architectures.
\subsubsection{Experimental Setup and Results}
% --------------------------------
%\paragraph{Models and baselines.} 
% --------------------------------
\paragraph{Dataset.}
We evaluated our approach on Visual Genome (VG) \citep{krishna2017visual}, a dataset with 108,077 images annotated with bounding boxes, entities and relations. On average, images have 12 entities and 7 relations per image.
%Entities and relations follow a long-tailed distribution, with a total of 75,729 unique entity classes and 40,480 unique relations. 
For a proper comparison with previous results \citep{pixels_to_graph,sg_generation_msg_pass,neural_motifs}, we used the data from \citep{sg_generation_msg_pass}, including the train and test splits. For evaluation, we used the same 150 entities and 50 relations as in \citep{pixels_to_graph,sg_generation_msg_pass,neural_motifs}.
To tune hyper-parameters, we also split the training data into two by randomly selecting 5K examples, resulting in a final 70K/5K/32K split for train/validation/test sets.

\begin{table}
\small
  \begin{tabular}{lcccccccc}
  \hline
  \multicolumn{1}{|c|}{} & \multicolumn{4}{c|}{Constrained Evaluation} & \multicolumn{4}{c|}{Unconstrained Evaluation}\\
        \multicolumn{1}{c|}{} & \multicolumn{2}{c|}{SGCls} & \multicolumn{2}{c|}{PredCls} &   \multicolumn{2}{c|}{SGCls} & \multicolumn{2}{c}{PredCls}\\
        & R@50 & R@100 & R@50 & R@100 & R@50 & R@100 & R@50 & R@100 \\
        \midrule
        Lu et al., 2016 \citep{lang_prior} & 11.8 & 14.1 & 35.0 & 27.9 & - & - & - & - \\
        Xu et al., 2017 \citep{sg_generation_msg_pass} & 21.7 & 24.4 & 44.8 & 53.0 & - & - & - & - \\
Pixel2Graph \citep{pixels_to_graph} & - & - & - & - & 26.5 & 30.0 & 68.0 & 75.2\\
{Graph R-CNN} \citep{graph_rcnn} & 29.6 & 31.6 & 54.2 & 59.1 & - & - & - & -\\
        Neural Motifs \citep{neural_motifs} & 35.8 & 36.5 & \textbf{65.2} & \textbf{67.1} & 44.5 & 47.7 & \textbf{81.1} & \textbf{88.3}\\
        Baseline \citep{neural_motifs} & 34.6 & 35.3 & 63.7 & 65.6 & 43.4 & 46.6 & 78.8 & 85.9\\
        \hline
        No Attention & 35.3 & 37.2 & 64.5 & 66.3 & 44.1 & 48.5 &  79.7 & 86.7\\ 
%        Neighbor Atten. & 35.4 & 37.6 & 64.5 & 66.4 \\
        Neighbor Attention & 35.7 & 38.5 & 64.6 & 66.6 & 44.7 & 49.9 &  80.0 & 87.1\\
        Linguistic & \textbf{36.5} & \textbf{38.8} & 65.1 & 66.9 & \textbf{45.5} & \textbf{50.8} & 80.8 & 88.2\\
        \bottomrule
  \hline
  \end{tabular}
  \caption{Test set results for graph-constrained evaluation (i.e., the returned triplets must be consistent with a scene graph) and for unconstrained evaluation (triplets need not be consistent with a scene graph).}
  \label{results}
\hspace{2.0cm}
\end{table}

\vspace{-5pt}
% -----------------------
\paragraph{Training.}
All networks were trained using Adam \citep{kingmaadam} with batch size $20$.
% \ignore{Input images were resized to 224x224 to conform with the \textsc{ResNet} architecture. We first trained the LP module, and then trained the SGP module using the best LP model, which was assessed on a validation set.}
Hyperparameter values below were chosen based on the validation set. 
%For LP, the loss is simply a cross-entropy on the logits corresponding $y_{i}$, with respect to the ground truth values for these variables (entities and relations both)
% Moreover, the relation network trained with a positive-to-negative ratio of 1:3 (where `positive' refers to a labeled relation and `negative' to unlabeled), and performed early-stopping after 90 epochs. We chose a batch size of 64, and also used data augmentation techniques such as translation and rotation to further improve the results.
The SGP loss function was the sum of cross-entropy losses over all entities and relations in the image. In the loss, we penalized entities $4$ times more strongly than relations, and penalized negative relations $10$ times more weakly than positive relations.

\vspace{-5pt}
% -----------------------
\paragraph{Evaluation.}
% -----------------------
In \citep{sg_generation_msg_pass} three different evaluation settings were considered. Here we focus on two of these:  
\textbf{(1) SGCls:} Given ground-truth bounding boxes for entities, predict all entity categories and relations categories. \textbf{(2) PredCls:} Given bounding boxes annotated with entity labels, predict all relations. Following \citep{lang_prior}, we used Recall@$K$ as the evaluation metric. It  measures the fraction of correct ground-truth triplets that appear within the $K$ most confident triplets proposed by the model.
 Two evaluation protocols are used in the literature differing in whether they enforce graph constraints over model predictions. The first {\em graph-constrained} protocol requires that the top-$K$ triplets assign one consistent class per entity and relation.
 %\footnote{This rules out putting more than one triplet for a pair of bounding boxes. It also rules out inconsistent assignment, like a bounding box that is labeled as one entity in one triplet, and as another entity in another triplet.} 
 The second {\em unconstrained} protocol does not enforce any such constraints. We report results on both protocols, following \citep{neural_motifs}.

\vspace{-5pt}
% --------------------------------
\paragraph{Models and baselines.} 
% --------------------------------
We compare four variants of our GPI approach with the reported results of four baselines that are currently the state-of-the-art on various scene graph prediction problems (all models use the same data split and pre-processing as \citep{sg_generation_msg_pass}): 1) \textsc{Lu et al., 2016 \citep{lang_prior}}: This work leverages word embeddings to fine-tune the likelihood of predicted relations. 2) \textsc{Xu et al, 2017 \citep{sg_generation_msg_pass}}: This model passes messages between entities and relations, and iteratively refines the feature map used for prediction. 3) \textsc{Newell \& Deng, 2017 \citep{pixels_to_graph}}: The \textsc{Pixel2Graph} model uses associative embeddings  \citep{associative_embedding} to produce a full graph from the image. {4) \textsc{Yang et al., 2018 \citep{graph_rcnn}}: The \textsc{GRAPH R-CNN} model uses object-relation regularities to sparsify and reason over scene graphs.} 5) \textsc{Zellers et al., 2017 \citep{neural_motifs}}: The \textsc{NeuralMotif} method encodes global context for capturing high-order motifs in scene graphs, and the \textsc{Baseline} outputs the entities and relations distributions without using the global context. The following variants of GPI were compared: 1) \textsc{GPI: No Attention}: Our GPI model, but with no attention mechanism. Instead, following \hyperref[graph_permutation_form] {Theorem 1}, we simply sum the features. 
2) \textsc{GPI: NeighborAttention}: Our GPI model, with attention over neighbors features. 
%3) \textsc{GPI: MultiAttention}: Our GPI model, except that we learn different attention weights per feature.
3) \textsc{GPI: Linguistic:} Same as \textsc{GPI: NeighborAttention} but also concatenating the word embedding vector, as described above.

\comment{
\begin{enumerate}[topsep=0pt,itemsep=0pt,parsep=0pt,partopsep=0pt]
\item 
\textsc{Lu et al., 2016 \citep{lang_prior}}: This work leverages word embeddings to fine-tune the likelihood of predicted relations.
\item \textsc{Xu et al, 2017 \citep{sg_generation_msg_pass}}: This model passes messages between entities and relations, and iteratively refines the feature map used for prediction.
\item \textsc{Newell \& Deng, 2017 \citep{pixels_to_graph}}. The \textsc{Pixel2Graph} model uses associative embeddings  \citep{associative_embedding} to produce a full graph from the image.
\item \textsc{Zellers et al., 2017 \citep{neural_motifs}} The \textsc{NeuralMotif} method encodes global context for capturing high-order motifs in scene graphs, and the \textsc{Baseline} outputs the entities and relations distributions without using the global context.
\item \textsc{GPI: No Attention}: Our GPI model, but with no attention mechanism. Instead, following \hyperref[graph_permutation_form] {Theorem 1}, we simply sum the features. 
\item \textsc{GPI: NeighborAttention}: Our GPI model, using attention over neighbors as described in Section \ref{sg_model}. 
\item \textsc{GPI: MultiAttention}: Our GPI model, except that we learn different attention weights per feature.
\item \textsc{GPI: Linguistic:} Same as \textsc{GPI: MultiAttention} but also concatenating the word embedding vector for the most probable entity label (see  Sec. \ref{sg_model}).
\end{enumerate}
}
%\subsubsection{Results}
% --------------------------------
\paragraph{Results.} 
% --------------------------------
Table~\ref{results} shows recall@$50$ and recall@$100$ for three variants of our approach, and compared with five baselines. All GPI variants performs well, with   \textsc{Linguistic} outperforming all baselines for SGCls and being comparable to the state-of-the-art model for PredCls. Note that PredCl is an easier task, which makes less use of the structure, hence it is  not surprising that GPI achieves similar accuracy to \cite{neural_motifs}. Figure \ref{fig:qualitative_results} illustrates the model behavior. Predicting isolated labels with $\ve_i$ (\ref{fig:qualitative_results}c) mislabels several entities, but these are corrected at the final output (\ref{fig:qualitative_results}d). Figure \ref{fig:qualitative_results}e shows that the system learned to attend more to nearby entities (the window and building are closer to the tree), and \ref{fig:qualitative_results}f shows that stronger attention is learned for the class bird, presumably because it is usually more informative than common classes like tree.

\begin{figure}
\begin{center}
\includegraphics[width=\linewidth]{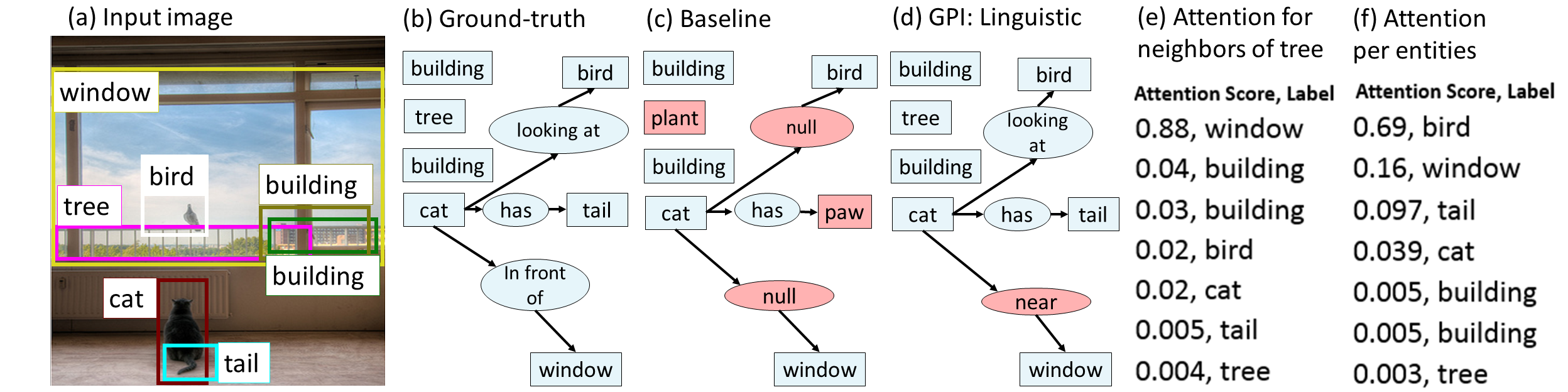}
\caption{\small{(a) An input image with bounding boxes from VG. (b) The ground-truth scene graph. (c) The Baseline fails to recognize some entities (\emph{tail} and \emph{tree}) and relations (\emph{in front of} instead of \emph{looking at}). (d) \textsc{GPI:Linguistic} fixes most incorrect LP predictions. (e)  \emph{Window} is the most significant neighbor of \emph{Tree}. (f) The entity \emph{bird} receives substantial attention, while \emph{Tree} and \emph{building} are less informative.}}
\label{fig:qualitative_results}
\end{center}
\end{figure}

\paragraph{{Implementation details.}}
The $\phi$ and $\alpha$ networks were each implemented as a single fully-connected (FC) layer with a 500-dimensional  outputs. $\rho$ was implemented as a FC network with 3 500-dimensional hidden layers, with one 150-dimensional output for the entity probabilities, and one 51-dimensional output for relation probabilities. The attention mechanism was implemented as a network like to $\phi$ and $\alpha$, receiving the same inputs, but using the output scores for the attention . The full code is available at https://github.com/shikorab/SceneGraph

\ignore{
\textbf{Gal:}
We further describe our Scene Graph Prediction (SGP) model. Recall that the model is composed of three sub-networks $\phi$, $\alpha$, $\rho$ and receives three types of inputs:   entity features $\ve_i$, relation features $\ve_j$ and entity-entity pairwise features $\ve_{i,j}$ which is the relation features between entity $i$ and entity $j$. 

The entity features include: (1) entity confidence learned by the baseline model. (2) bounding box given as a \texttt{(x, y, width, height)}. (3) number of entities with smaller/bigger size. (4) number of entities to the left/right/above/below. (5) number of entities with higher/lower confidence. (6) For Linguistic model only: word embedding of the most probable class. Word vectors were learned with \textsc{GloVe} from the ground-truth captions of Visual Genome.

The relation features include: (1) relation confidence learned by the baseline model. (2) For the Linguistic model only: word embedding of the most probable class. 
}

\vspace{-5pt}
\section{Conclusion}
\label{conclusion}
\vspace{-5pt}
We presented a deep learning approach to structured prediction, which constrains the architecture
to be invariant to structurally identical inputs. As in score-based methods,
our approach relies on pairwise features, capable of describing inter-label correlations,  and thus inheriting the intuitive aspect of score-based approaches. However, instead of maximizing a score function (which leads to computationally-hard inference), we directly produce an output that is invariant to equivalent representations of the pairwise terms.

This axiomatic approach to model architecture can be extended in many ways. For image labeling, geometric invariances (shift or rotation) may be desired. In other cases, invariance to feature permutations may be desirable. We leave the derivation of the corresponding architectures to future work. Finally, there may be cases where the invariant structure is unknown and should be discovered from data, which is related to work on lifting graphical models \cite{Bui:2013}. It would be interesting to explore algorithms that discover and use such symmetries for deep structured prediction.
\section*{Acknowledgements}
\vspace{-5pt}
This work was supported by the ISF Centers of Excellence grant, and by the Yandex Initiative in Machine Learning. Work by GC was performed while at Google Brain Research.

\bibliography{papers_bib}

\begin{thebibliography}{33}
\providecommand{\natexlab}[1]{#1}
\providecommand{\url}[1]{\texttt{#1}}
\expandafter\ifx\csname urlstyle\endcsname\relax
  \providecommand{\doi}[1]{doi: #1}\else
  \providecommand{\doi}{doi: \begingroup \urlstyle{rm}\Url}\fi

\bibitem[Bahdanau et~al.(2015)Bahdanau, Cho, and Bengio]{bahdanau2015neural}
Bahdanau, D., Cho, K., and Bengio, Y.
\newblock Neural machine translation by jointly learning to align and
  translate.
\newblock In \emph{International Conference on Learning Representations
  (ICLR)}, 2015.

\bibitem[Belanger et~al.(2017)Belanger, Yang, and McCallum]{belanger17a}
Belanger, David, Yang, Bishan, and McCallum, Andrew.
\newblock End-to-end learning for structured prediction energy networks.
\newblock In Precup, Doina and Teh, Yee~Whye (eds.), \emph{Proceedings of the
  34th International Conference on Machine Learning}, volume~70, pp.\
  429--439. PMLR, 2017.

\bibitem[Bello et~al.(2016)Bello, Pham, Le, Norouzi, and
  Bengio]{bello2016neural}
Bello, Irwan, Pham, Hieu, Le, Quoc~V, Norouzi, Mohammad, and Bengio, Samy.
\newblock Neural combinatorial optimization with reinforcement learning.
\newblock \emph{arXiv preprint arXiv:1611.09940}, 2016.

\bibitem[Bui et~al.(2013)Bui, Huynh, and Riedel]{Bui:2013}
Bui, Hung~Hai, Huynh, Tuyen~N., and Riedel, Sebastian.
\newblock Automorphism groups of graphical models and lifted variational
  inference.
\newblock In \emph{Proceedings of the Twenty-Ninth Conference on Uncertainty in
  Artificial Intelligence}, UAI'13, pp.\  132--141, 2013.

\bibitem[Chen \& Manning(2014)Chen and Manning]{chen2014fast}
Chen, Danqi and Manning, Christopher.
\newblock A fast and accurate dependency parser using neural networks.
\newblock In \emph{Proceedings of the 2014 conference on empirical methods in
  natural language processing (EMNLP)}, pp.\  740--750, 2014.

\bibitem[Chen et~al.(2014)Chen, Papandreou, Kokkinos, Murphy, and
  Yuille]{chen2014semantic}
Chen, Liang~Chieh, Papandreou, George, Kokkinos, Iasonas, Murphy, Kevin, and
  Yuille, Alan~L.
\newblock Semantic image segmentation with deep convolutional nets and fully
  connected {CRFs}.
\newblock In \emph{Proceedings of the Second International Conference on
  Learning Representations}, 2014.

\bibitem[Chen et~al.(2015)Chen, Schwing, Yuille, and Urtasun]{chen2015learning}
Chen, Liang~Chieh, Schwing, Alexander~G, Yuille, Alan~L, and Urtasun, Raquel.
\newblock Learning deep structured models.
\newblock In \emph{Proc. ICML}, 2015.

\bibitem[Farabet et~al.(2013)Farabet, Couprie, Najman, and
  LeCun]{farabet2013learning}
Farabet, Clement, Couprie, Camille, Najman, Laurent, and LeCun, Yann.
\newblock Learning hierarchical features for scene labeling.
\newblock \emph{IEEE transactions on pattern analysis and machine
  intelligence}, 35\penalty0 (8):\penalty0 1915--1929, 2013.

\bibitem[Gilmer et~al.(2017)Gilmer, Schoenholz, Riley, Vinyals, and
  Dahl]{gilmer2017neural}
Gilmer, Justin, Schoenholz, Samuel~S, Riley, Patrick~F, Vinyals, Oriol, and
  Dahl, George~E.
\newblock Neural message passing for quantum chemistry.
\newblock \emph{arXiv preprint arXiv:1704.01212}, 2017.

\bibitem[Gygli et~al.(2017)Gygli, Norouzi, and Angelova]{pmlr-v70-gygli17a}
Gygli, Michael, Norouzi, Mohammad, and Angelova, Anelia.
\newblock Deep value networks learn to evaluate and iteratively refine
  structured outputs.
\newblock In Precup, Doina and Teh, Yee~Whye (eds.), \emph{Proceedings of the
  34th International Conference on Machine Learning}, volume~70 of
  \emph{Proceedings of Machine Learning Research}, pp.\  1341--1351,
  International Convention Centre, Sydney, Australia, 2017. PMLR.

\bibitem[Johnson et~al.(2015)Johnson, Krishna, Stark, Li, Shamma, Bernstein,
  and Li]{img_retriev_using_sg}
Johnson, Justin, Krishna, Ranjay, Stark, Michael, Li, Li{-}Jia, Shamma,
  David~A., Bernstein, Michael~S., and Li, Fei{-}Fei.
\newblock Image retrieval using scene graphs.
\newblock In \emph{Proc. Conf. Comput. Vision Pattern Recognition}, pp.\
  3668--3678, 2015.

\bibitem[Johnson et~al.(2018)Johnson, Gupta, and Fei{-}Fei]{johnson2018image}
Johnson, Justin, Gupta, Agrim, and Fei{-}Fei, Li.
\newblock Image generation from scene graphs.
\newblock \emph{arXiv preprint arXiv:1804.01622}, 2018.

\bibitem[Khalil et~al.(2017)Khalil, Dai, Zhang, Dilkina, and
  Song]{khalil2017learning}
Khalil, Elias, Dai, Hanjun, Zhang, Yuyu, Dilkina, Bistra, and Song, Le.
\newblock Learning combinatorial optimization algorithms over graphs.
\newblock In \emph{Advances in Neural Information Processing Systems}, pp.\
  6351--6361, 2017.

\bibitem[Kingma \& Ba(2014)Kingma and Ba]{kingmaadam}
Kingma, Diederik~P. and Ba, Jimmy.
\newblock Adam: {A} method for stochastic optimization.
\newblock \emph{arXiv preprint arXiv: 1412.6980}, abs/1412.6980, 2014.

\bibitem[Krishna et~al.(2017)Krishna, Zhu, Groth, Johnson, Hata, Kravitz, Chen,
  Kalantidis, Li, Shamma, et~al.]{krishna2017visual}
Krishna, Ranjay, Zhu, Yuke, Groth, Oliver, Johnson, Justin, Hata, Kenji,
  Kravitz, Joshua, Chen, Stephanie, Kalantidis, Yannis, Li, Li-Jia, Shamma,
  David~A, et~al.
\newblock Visual genome: Connecting language and vision using crowdsourced
  dense image annotations.
\newblock \emph{International Journal of Computer Vision}, 123\penalty0
  (1):\penalty0 32--73, 2017.

\bibitem[Lafferty et~al.(2001)Lafferty, McCallum, and
  Pereira]{Lafferty01conditional}
Lafferty, J., McCallum, A., and Pereira, F.
\newblock Conditional random fields: {P}robabilistic models for segmenting and
  labeling sequence data.
\newblock In \emph{Proceedings of the 18th International Conference on Machine
  Learning}, pp.\  282--289, 2001.

\bibitem[Liao et~al.(2016)Liao, Yang, Ackermann, and
  Rosenhahn]{support_relations}
Liao, Wentong, Yang, Michael~Ying, Ackermann, Hanno, and Rosenhahn, Bodo.
\newblock On support relations and semantic scene graphs.
\newblock arXiv preprint arXiv:1609.05834, 2016.

\bibitem[Lin et~al.(2015)Lin, Shen, Reid, and van~den Hengel]{lin2015deeply}
Lin, Guosheng, Shen, Chunhua, Reid, Ian, and van~den Hengel, Anton.
\newblock Deeply learning the messages in message passing inference.
\newblock In \emph{Advances in Neural Information Processing Systems}, pp.\
  361--369, 2015.

\bibitem[Lu et~al.(2016)Lu, Krishna, Bernstein, and Li]{lang_prior}
Lu, Cewu, Krishna, Ranjay, Bernstein, Michael~S., and Li, Fei{-}Fei.
\newblock Visual relationship detection with language priors.
\newblock In \emph{European Conf. Comput. Vision}, pp.\  852--869, 2016.

\bibitem[Meshi et~al.(2010)Meshi, Sontag, Jaakkola, and Globerson]{Meshi10}
Meshi, O., Sontag, D., Jaakkola, T., and Globerson, A.
\newblock Learning efficiently with approximate inference via dual losses.
\newblock In \emph{Proceedings of the 27th International Conference on Machine
  Learning}, pp.\  783--790, New York, NY, USA, 2010. ACM.

\bibitem[Newell \& Deng(2017)Newell and Deng]{pixels_to_graph}
Newell, Alejandro and Deng, Jia.
\newblock Pixels to graphs by associative embedding.
\newblock In \emph{Advances in Neural Information Processing Systems 30 (to
  appear)}, pp.\  1172--1180. Curran Associates, Inc., 2017.

\bibitem[Newell et~al.(2017)Newell, Huang, and Deng]{associative_embedding}
Newell, Alejandro, Huang, Zhiao, and Deng, Jia.
\newblock Associative embedding: End-to-end learning for joint detection and
  grouping.
\newblock In \emph{Neural Inform. Process. Syst.}, pp.\  2274--2284. Curran
  Associates, Inc., 2017.

\bibitem[Pei et~al.(2015)Pei, Ge, and Chang]{PeiGC15}
Pei, Wenzhe, Ge, Tao, and Chang, Baobao.
\newblock An effective neural network model for graph-based dependency parsing.
\newblock In \emph{Proceedings of the 53rd Annual Meeting of the Association
  for Computationa Linguistics}, pp.\  313--322, 2015.

\bibitem[Plummer et~al.(2017)Plummer, Mallya, Cervantes, Hockenmaier, and
  Lazebnik]{plummerPLCLC2017}
Plummer, Bryan~A., Mallya, Arun, Cervantes, Christopher~M., Hockenmaier, Julia,
  and Lazebnik, Svetlana.
\newblock Phrase localization and visual relationship detection with
  comprehensive image-language cues.
\newblock In \emph{ICCV}, pp.\  1946--1955, 2017.

\bibitem[Raposo et~al.(2017)Raposo, Santoro, Barrett, Pascanu, Lillicrap, and
  Battaglia]{entangled_scene}
Raposo, David, Santoro, Adam, Barrett, David, Pascanu, Razvan, Lillicrap,
  Timothy, and Battaglia, Peter.
\newblock Discovering objects and their relations from entangled scene
  representations.
\newblock arXiv preprint arXiv:1702.05068, 2017.

\bibitem[Schwing \& Urtasun(2015)Schwing and Urtasun]{schwing2015fully}
Schwing, Alexander~G and Urtasun, Raquel.
\newblock Fully connected deep structured networks.
\newblock \emph{ArXiv e-prints}, 2015.

\bibitem[Shelhamer et~al.(2017)Shelhamer, Long, and Darrell]{fcn}
Shelhamer, Evan, Long, Jonathan, and Darrell, Trevor.
\newblock Fully convolutional networks for semantic segmentation.
\newblock \emph{Proc. Conf. Comput. Vision Pattern Recognition}, 39\penalty0
  (4):\penalty0 640--651, 2017.

\bibitem[Taskar et~al.(2004)Taskar, Guestrin, and Koller]{taskar03max}
Taskar, B., Guestrin, C., and Koller, D.
\newblock Max margin {M}arkov networks.
\newblock In Thrun, S., Saul, L., and {Sch\"{o}lkopf}, B. (eds.),
  \emph{Advances in Neural Information Processing Systems 16}, pp.\  25--32.
  MIT Press, Cambridge, MA, 2004.

\bibitem[Xu et~al.(2017)Xu, Zhu, Choy, and Fei{-}Fei]{sg_generation_msg_pass}
Xu, Danfei, Zhu, Yuke, Choy, Christopher~B., and Fei{-}Fei, Li.
\newblock {Scene Graph Generation by Iterative Message Passing}.
\newblock In \emph{Proc. Conf. Comput. Vision Pattern Recognition}, pp.\
  3097--3106, 2017.

\bibitem[Yang et~al.(2018)Yang, Lu, Lee, Batra, and Parikh]{graph_rcnn}
Yang, Jianwei, Lu, Jiasen, Lee, Stefan, Batra, Dhruv, and Parikh, Devi.
\newblock Graph {R-CNN} for scene graph generation.
\newblock In \emph{European Conf. Comput. Vision}, pp.\  690--706, 2018.

\bibitem[Zaheer et~al.(2017)Zaheer, Kottur, Ravanbakhsh, Poczos, Salakhutdinov,
  and Smola]{deep_sets}
Zaheer, Manzil, Kottur, Satwik, Ravanbakhsh, Siamak, Poczos, Barnabas,
  Salakhutdinov, Ruslan~R, and Smola, Alexander~J.
\newblock Deep sets.
\newblock In \emph{Advances in Neural Information Processing Systems 30}, pp.\
  3394--3404. Curran Associates, Inc., 2017.

\bibitem[Zellers et~al.(2017)Zellers, Yatskar, Thomson, and
  Choi]{neural_motifs}
Zellers, Rowan, Yatskar, Mark, Thomson, Sam, and Choi, Yejin.
\newblock Neural motifs: Scene graph parsing with global context.
\newblock \emph{arXiv preprint arXiv:1711.06640}, abs/1711.06640, 2017.

\bibitem[Zheng et~al.(2015)Zheng, Jayasumana, Romera-Paredes, Vineet, Su, Du,
  Huang, and Torr]{zheng2015conditional}
Zheng, Shuai, Jayasumana, Sadeep, Romera-Paredes, Bernardino, Vineet, Vibhav,
  Su, Zhizhong, Du, Dalong, Huang, Chang, and Torr, Philip~HS.
\newblock Conditional random fields as recurrent neural networks.
\newblock In \emph{Proceedings of the IEEE International Conference on Computer
  Vision}, pp.\  1529--1537, 2015.

\end{thebibliography}
\bibliographystyle{icml2018}
% \bibliographystyle{plain}

% \newpage
\section{Supplementary Material}

This supplementary material includes: (1) Visual illustration of the proof of Theorem 1. (2) Explaining how to integrate an attention mechanism in our GPI framework. (3) Additional evaluation method to further analyze and compare our work with baselines.

\subsection{Theorem 1: Illustration of Proof}

\begin{figure}[h]
    \begin{center}
  \centerline{\includegraphics[width=0.8\columnwidth]{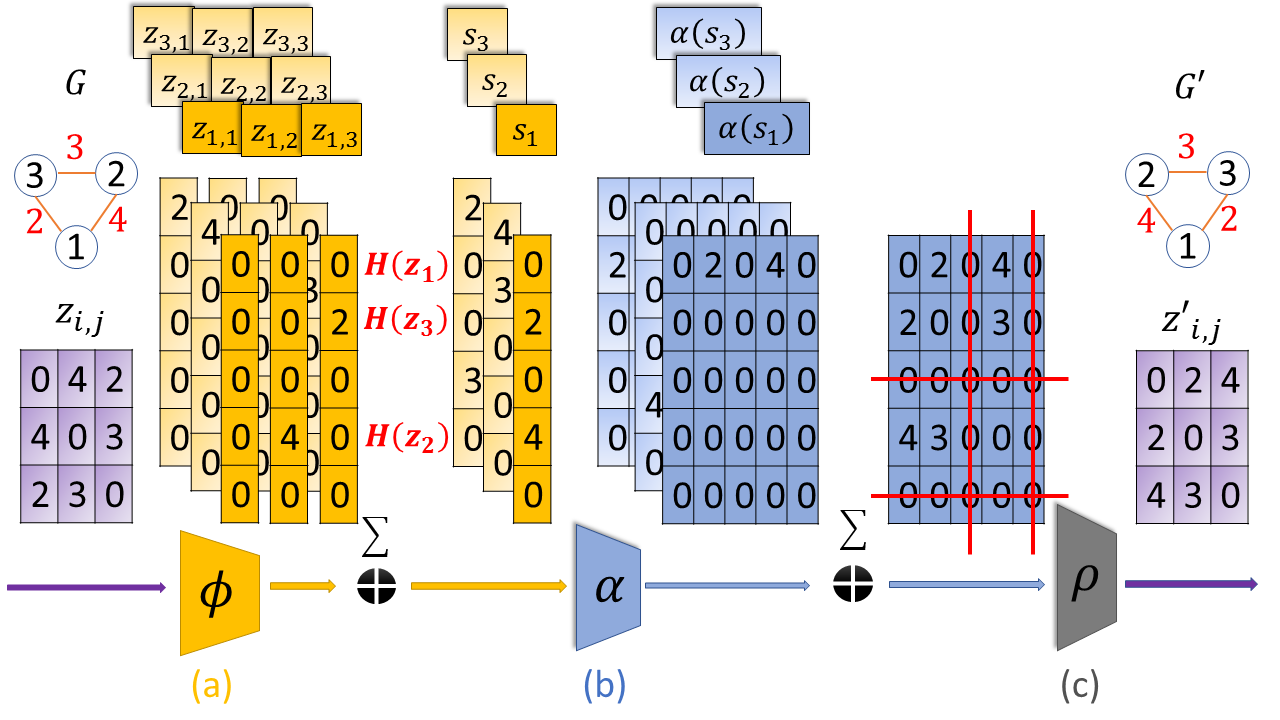}}
 % \vspace{-10pt}
    \caption{{\small Illustration of the proof of Theorem 1 using a specific construction example. Here $H$ is a hash function of size $L$ = 5 such that $H(1)=1,H(3)=2,H(2)=4$, $G$ is a three-node input graph, and $\ve_{i,j}\in\reals$ are the pairwise features (in purple) of $G$. \textbf{(a)} $\phi$ is applied to each $\ve_{i,j}$. Each application yields a vector in $\reals^5$. The three dark yellow columns correspond to $\phi(\ve_{1,1})$, $\phi(\ve_{1,2})$ and $\phi(\ve_{1,3})$. Then, all vectors $\phi(\ve_{i,j})$ are summed over $j$ to obtain three $\ss_i$ vectors. \textbf{(b)} $\alphav$'s (blue matrices) are an outer product between $\onehotvector{H(\ve_i)}$ and $\ss_i$ resulting in a matrix of zeros except one row. The dark blue matrix corresponds for $\alpha(\ve_1,\ss_1)$. \textbf{(c)} All $\alpha$'s are summed to a $5 \times 5$ matrix, isomorphic to the original $\ve_{i,j}$ matrix.}}
    \label{hash_graph}
    \end{center}
\end{figure}

\subsection{Characterizing Permutation Invariance: Attention}
Attention is a powerful component which naturally can be introduced into our GPI model. We now show how attention can be introduced in our framework.
Formally, we learn attention weights for the neighbors $j$ of a node $i$, which scale the features $\ve_{i,j}$ of that neighbor. We can also learn different attention weights for individual features of each neighbor in a similar way.

\newcommand{\zzi}{{\zz_i}}
\newcommand{\zzj}{{\zz_j}}
\newcommand{\zzn}{{\zz_n}}
\newcommand{\zzio}{{\zz_{i,1}}}
\newcommand{\zzin}{{\zz_{i,n}}}
\newcommand{\zzij}{{\zz_{i,j}}}

Let $w_{i,j}\in \reals$ be an attention mask specifying the weight that node $i$ gives to node $j$:
%$\mathbf{w}_{i,j}$ is in the simplex $\Delta^{2n}$
\begin{equation}
\label{eq:att_softmax}
	w_{i,j}(\zzi, \zzij, \zzj) = {e^{\beta(\zzi,\zzij,\zzj)}} / {\sum_{t} e^{\beta(\zzi,\zz_{i,t},\zz_{t})}}
\end{equation}
where $\beta$ can be any scalar-valued function of its arguments (e.g., a dot product of $\zz_i$ and $\zz_j$ as in standard attention models). \ignore{Alternative}
To introduce attention we wish $\alphav\in \reals^e$ to have the form of weighting $w_{i,j}$ over neighboring feature vectors $\ve_{i,j}$, namely, $\alphav = \sum_{j \neq i} {w}_{i,j} \zz_{i,j}$. 

To achieve this form we extend $\phiv$ by a single entry, defining $\phiv\in\reals^{e+1}$ (namely we set $L=e+1$) as $\phiv_{1:e}(\ve_i, \ve_{i,j}, \ve_j) =e^{\beta(\ve_i, \ve_{i,j}, \ve_j)}\ve_{i,j}$ (here $\phiv_{1:e}$ are the first $e$ elements of $\phiv$) and $\phiv_{e+1}(\ve_i, \ve_{i,j}; \ve_j) =e^{\beta(\ve_i, \ve_{i,j}, \ve_j)}$.
We keep the definition of $\ss_i=\sum_{j\neq i} \boldsymbol{\phi}(\ve_{i}, \vescalar_{i,j}, \ve_{j})$.
Next, we define $\alphav=\frac{\ss_{i,1:e}}{\ss_{i,e+1}}$ and substitute $\ss_i$ and $\phiv$ to obtain the desired form as attention weights $w_{i,j}$ over neighboring feature vectors $\ve_{i,j}$: 
\begin{align*}
\alphav(\zz_i, \ss_i) \!=\! \frac{\ss_{i,1:e}}{\ss_{i,e+1}}
\!=\! \sum_{j \neq i} \frac{e^{\beta(\zzi,\zzij,\zzj)}\ve_{i,j}}{\sum_{j \neq i} e^{\beta(\zzi,\zzij,\zzj)}}  \!=\! \sum_{j \neq i} {w}_{i,j} \zz_{i,j}
%[\zz_{i,j}; \zz_j].
\end{align*}
A similar approach can be applied over $\alphav$ and $\rho$ to model attention over the outputs of $\alphav$ as well (graph nodes).

\subsection{Scene Graph Results}
In the main paper, we described the results for the two prediction tasks: SGCls and PredCls, as defined in section 5.2.1: "Experimental Setup and Results". To further analyze our module, we compare the best variant, \textsc{GPI: Linguistic}, per relation to two baselines: \citep{lang_prior} and \cite{sg_generation_msg_pass}. Table \ref{result_predicate_table}, specifies the  PredCls recall@5 of the 20-top frequent relation classes. The GPI module performs better in almost all the relations classes.

\begin{table}[ht!]
  \caption{Recall@$5$ of PredCls for the 20-top relations ranked by their frequency, as in  \citep{sg_generation_msg_pass}}\label{result_predicate_table}
  \vspace{0.0in}
  \begin{center}
  \begin{small}
  \begin{sc}
  \begin{tabular}{lllll}
  Relation & \citep{lang_prior} & \citep{sg_generation_msg_pass} & Linguistic\\
  \midrule
  on & \textbf{99.71}  & 99.25 & 99.3 \\
  has & 98.03  & 97.25 & \textbf{98.7} \\
  in & 80.38 & 88.30  & \textbf{95.9} \\
  of & 82.47 & 96.75  & \textbf{98.1} \\
  wearing & 98.47 & 98.23  & \textbf{99.6} \\
  near & 85.16 & \textbf{96.81}  & 95.4 \\
  with & 31.85 & 88.10  & \textbf{94.2} \\
  above & 49.19 & 79.73  & \textbf{83.9} \\
  holding & 61.50 & 80.67  & \textbf{95.5} \\
  behind & 79.35 & \textbf{92.32}  & 91.2 \\
  under & 28.64 & 52.73  & \textbf{83.2} \\
  sitting on & 31.74 & 50.17  & \textbf{90.4} \\
  in front of & 26.09 & 59.63  & \textbf{74.9} \\
  attached to & 8.45 & 29.58  & \textbf{77.4} \\
  at & 54.08 & 70.41  & \textbf{80.9} \\
  hanging from & 0.0 & 0.0  & \textbf{74.1} \\
  over & 9.26 & 0.0  & \textbf{62.4}  \\
  for & 12.20 & 31.71  & \textbf{45.1} \\
  riding & 72.43 & 89.72  & \textbf{96.1} \\
  \bottomrule
  \end{tabular}  \end{sc}  \end{small}  \end{center}
  \vskip -0.1in
\end{table}

\end{document}